\documentclass[journal, twocolumn]{IEEEtran}

\usepackage{amsmath, amssymb, amsthm} 
\newtheorem{theorem}{Theorem}
\newtheorem{lemma}{Lemma}

\newtheorem{problem}{Problem}

\newtheorem{definition}{Definition}
\newtheorem{corollary}{Corollary}

\usepackage[pdftex]{graphicx}

\usepackage{array}
\usepackage{cite}
\usepackage{bm}

\begin{document}

\title{On the uniqueness and stability of dictionaries \\ for sparse representation of noisy signals}

\author{Charles~J.~Garfinkle and Christopher~J.~Hillar \\
Redwood Center for Theoretical Neuroscience, Berkeley, CA, USA
}

\maketitle


\begin{abstract}
Learning optimal dictionaries for sparse coding has exposed characteristic sparse features of many natural signals. However, universal guarantees of the stability of such features in the presence of noise are lacking. Here, we provide very general conditions guaranteeing when dictionaries yielding the sparsest encodings are unique and stable with respect to measurement or modeling error. We demonstrate that some or all original dictionary elements are recoverable from noisy data even if the dictionary fails to satisfy the spark condition, its size is overestimated, or only a polynomial number of distinct sparse supports appear in the data. Importantly, we derive these guarantees without requiring any constraints on the recovered dictionary beyond a natural upper bound on its size. Our results also yield an effective procedure sufficient to affirm if a proposed solution to the dictionary learning problem is unique within bounds commensurate with the noise. We suggest applications to data analysis, engineering, and neuroscience and close with some remaining challenges left open by our work.
\end{abstract}


\section{Introduction}\label{Intro}
\IEEEPARstart{S}{parse} coding is a common modern approach to pattern analysis in signal processing whereby each of $N$ observed $n$-dimensional signal samples is viewed as a (noisy) linear combination of at most $k$ elementary waveforms drawn from some unknown ``dictionary" of size $m \ll N$ (see \cite{Zhang15} for a comprehensive review). 
Optimizing dictionaries subject to this and related sparsity constraints has revealed seemingly characteristic sparse structure in several signal classes of current interest (e.g., in vision \cite{wang2015sparse}). 

Of particular note are the seminal works in the field \cite{Olshausen96, hurri1996image, bell1997independent, van1998independent}, which discovered that dictionaries optimized for coding small patches of ``natural" images share qualitative similarities with linear filters estimated from response properties of simple-cell neurons in mammalian visual cortex. Curiously, these waveforms (e.g., ``Gabor'' wavelets) appear in dictionaries learned by a variety of algorithms trained over different natural image datasets, suggesting that learned features in natural signals may, in some sense, be canonical \cite{donoho2001can}.

Motivated by these discoveries and more recent work relating compressed sensing \cite{eldar2012compressed} to a theory of information transmission through random wiring bottlenecks in the brain \cite{Isely10}, we address when dictionaries for sparse representation are indeed identifiable from data. Answers to this question may also have implications in practice wherever an appeal is made to latent sparse structure of data (e.g., forgery detection \cite{hughes2010, olshausen2010applied}; brain recordings \cite{jung2001imaging, agarwal2014spatially, lee2016sparse}; and gene expression \cite{wu2016stability}). 

While several algorithms have been recently proposed to provably recover unique dictionaries under specific conditions (see \cite[Sec.~I-E]{Sun16} for a summary of the state-of-the-art), few theorems can be invoked to justify the consistency of inference under this model of data more broadly. To our knowledge, a universal guarantee of the uniqueness and stability of learned dictionaries and the sparse representations they induce over noisy data has yet to appear in the literature.

Here, we prove very generally that uniqueness and stability is a typical property of sparse dictionary learning. More specifically, we show that matrices injective on a sparse domain are identifiable from \mbox{$N = m(k-1){m \choose k} + m$} noisy linear combinations of $k$ of their $m$ columns up to an error that is linear in the noise (Thm.~\ref{DeterministicUniquenessTheorem}). In fact, provided $n \geq \min(2k,m)$, in almost all cases the problem is well-posed, as per Hadamard \cite{Hadamard1902}, given a sufficient amount of data (Thm.~\ref{robustPolythm} and Cor.~\ref{ProbabilisticCor}). 

Our guarantees also hold for a related (and perhaps more commonly posed, e.g. \cite{rehnsommer2007}) optimization problem seeking a dictionary minimizing the average number of elementary waveforms required to reconstruct each sample of the dataset (Thm.~\ref{SLCopt}). To practical benefit, our results impose no restrictions on learned dictionaries (e.g., that they, too, be injective over some sparse domain) beyond an upper bound on dictionary size, which is necessary in any case to avoid trivial solutions (e.g., allowing $m = N$). 

More precisely, let $\mathbf{A} \in \mathbb R^{n \times m}$ be a matrix with columns $\mathbf{A}_j$ ($j = 1,\ldots,m$) and let dataset $Z$ consist of measurements:
\begin{align}\label{LinearModel}
\mathbf{z}_i = \mathbf{A}\mathbf{x}_i + \mathbf{n}_i,\ \ \  \text{$i=1,\ldots,N$},
\end{align}
for $k$-\emph{sparse} $\mathbf{x}_i \in \mathbb{R}^m$ having at most $k<m$ nonzero entries and \emph{noise} $\mathbf{n}_i \in \mathbb{R}^n$, with bounded norm $\| \mathbf{n}_i \|_2 \leq  \eta$ representing our worst-case uncertainty in measuring the product $\mathbf{A}\mathbf{x}_i$. We first consider the following formulation of sparse coding.

\begin{problem}\label{InverseProblem}
Find a dictionary matrix $\mathbf{B}$ and $k$-sparse codes $\mathbf{\overline x}_1, \ldots, \mathbf{\overline x}_N$ that satisfy $\|\mathbf{z}_i - \mathbf{B}\mathbf{\overline x}_i\|_2 \leq \eta$ for all $i = 1,\ldots,N$.
\end{problem}

Note that every solution to Prob.~\ref{InverseProblem} represents infinitely many equivalent alternatives $\mathbf{BPD}$ and $\mathbf{D}^{-1}\mathbf{P}^{\top}\mathbf{\overline x}_1, \ldots, \mathbf{D}^{-1}\mathbf{P}^{\top}\mathbf{\overline x}_N$ parametrized by a choice of permutation matrix $\mathbf{P}$ and invertible diagonal matrix $\mathbf{D}$. 
Identifying these ambiguities (labelling and scale) yields a single orbit of solutions represented by any particular set of elementary waveforms (the columns of $\mathbf{B}$) and their associated sparse coefficients (the entries of $\mathbf{\overline x}_i$) that reconstruct each data point $\mathbf{z}_i$. 

Previous theoretical work addressing the noiseless case $\eta =0$ (e.g., \cite{li2004analysis, Georgiev05, Aharon06, Hillar15}) for matrices $\mathbf{B}$ having exactly $m$ columns has shown that a solution to Prob.~\ref{InverseProblem}, when it exists, is unique up to such relabeling and rescaling provided the $\mathbf{x}_i$ are sufficiently diverse and $\mathbf{A}$ satisfies the \textit{spark condition}:
\begin{align}\label{SparkCondition}
\mathbf{A}\mathbf{x}_1 = \mathbf{A}\mathbf{x}_2 \implies \mathbf{x}_1 = \mathbf{x}_2, \ \ \ \text{for all $k$-sparse } \mathbf{x}_1, \mathbf{x}_2,
\end{align}
which is necessary to guarantee the uniqueness of arbitrary $k$-sparse $\mathbf{x}_i$. We generalize these results to the practical setting  $\eta > 0$ by considering the following natural notion of stability with respect to measurement error.

\begin{definition}\label{maindef}
Fix $Y = \{ \mathbf{y}_1, \ldots, \mathbf{y}_N\} \subset \mathbb{R}^n$. We say $Y$ has a \textbf{$k$-sparse representation in $\mathbb{R}^m$} if there exists a matrix $\mathbf{A}$ and $k$-sparse $\mathbf{x}_1, \ldots, \mathbf{x}_N \in \mathbb{R}^m$ such that $\mathbf{y}_i = \mathbf{A}\mathbf{x}_i$ for all $i$. 
This representation is \textbf{stable} if for every $\delta_1, \delta_2 \geq 0$, there exists some $\varepsilon = \varepsilon(\delta_1, \delta_2)$ that is strictly positive for positive $\delta_1$ and $\delta_2$ such that if $\mathbf{B}$ and $k$-sparse $\mathbf{\overline x}_1, \ldots, \mathbf{\overline x}_N \in \mathbb{R}^m$ satisfy:
\begin{align*}
	\|\mathbf{A}\mathbf{x}_i - \mathbf{B}\mathbf{\overline x}_i\|_2 \leq \varepsilon(\delta_1, \delta_2),\ \   \text{for all $i=1,\ldots,N$},
\end{align*}
then there is some permutation matrix $\mathbf{P}$ and invertible diagonal matrix $\mathbf{D}$ such that for all $i, j$:
\begin{align}\label{def1}
\|\mathbf{A}_j - (\mathbf{BPD})_j\|_2 \leq \delta_1 \ \text{and} \ \|\mathbf{x}_i - \mathbf{D}^{-1}\mathbf{P}^{\top}\mathbf{\overline x}_i\|_1 \leq \delta_2.
\end{align}
\end{definition}

To see how Prob. \ref{InverseProblem} motivates Def. \ref{maindef}, suppose that $Y$ has a stable $k$-sparse representation in $\mathbb{R}^m$ and fix $\delta_1, \delta_2$ to be the desired accuracies of recovery in \eqref{def1}. Consider any dataset $Z$ generated as in \eqref{LinearModel} with $\eta \leq \frac{1}{2} \varepsilon(\delta_1, \delta_2)$. Using the triangle inequality, it follows that any $n \times m$ matrix $\mathbf{B}$ and $k$-sparse $\mathbf{\overline x}_1, \ldots, \mathbf{\overline x}_N$ solving Prob.~\ref{InverseProblem} are necessarily within $\delta_1$ and $\delta_2$ of the original dictionary $\mathbf{A}$ and codes $\mathbf{x}_1, \ldots, \mathbf{x}_N$, respectively.\footnote{We mention that the different norms in \eqref{def1} reflect the distinct meanings typically ascribed to the dictionary and sparse codes in modeling data.}

The main result of this work is a very general uniqueness theorem for sparse coding (Thm.~\ref{DeterministicUniquenessTheorem}) directly 
implying (Cor.~\ref{DeterministicUniquenessCorollary}), which guarantees that sparse representations of a dataset $Z$ are unique up to noise whenever generating dictionaries $\mathbf{A}$ satisfy a spark condition on supports and the original sparse codes $\mathbf{x}_i$ are sufficiently diverse (e.g., Fig.~\ref{noisyrecovery}).  Moreover, we provide an explicit, computable $\varepsilon(\delta_1, \delta_2)$ in (\ref{epsdel}) that is linear in desired accuracy $\delta_1$, and essentially so in $\delta_2$.

In the next section, we give formal statements of these findings.  We then extend the same guarantees (Thm.~\ref{SLCopt}) to the following alternate formulation of dictionary learning, which minimizes the total number of nonzero entries in sparse codes.

\begin{problem}\label{OptimizationProblem}
Find matrices $\mathbf{B}$ and vectors \mbox{$\mathbf{\overline x}_1, \ldots, \mathbf{\overline x}_N$} solving:
\begin{align}\label{minsum}
\min \sum_{i = 1}^N \|\mathbf{\overline x}_{i}\|_0 \ \
\text{subject to} \ \ \|\mathbf{z}_i - \mathbf{B}\mathbf{\overline x}_i\|_2 \leq \eta, \ \text{for all $i$}.
\end{align}
\end{problem}

Our development of Thm.~\ref{DeterministicUniquenessTheorem} is general enough to provide some uniqueness and stability even when generating $\mathbf{A}$ do not fully satisfy (\ref{SparkCondition}) and recovery dictionaries $\mathbf{B}$ have more columns than $\mathbf{A}$.  Moreover, the approach incorporates a combinatorial theory for designing generating codes that should be of independent interest. We also give brief arguments adapting our results to dictionaries and codes drawn from probability distributions (Cor.~\ref{ProbabilisticCor}). The technical proofs of Thms.~\ref{DeterministicUniquenessTheorem} and ~\ref{SLCopt} are deferred to Sec.~\ref{DUT}, following some necessary definitions and a fact in combinatorial matrix analysis (Lem.~\ref{MainLemma}; proven in the Appendix). Finally, we discuss in Sec.~\ref{Discussion} applications of our mathematical observations as well as open questions. 

\section{Results}\label{Results}


Precise statements of our results require that we first identify some combinatorial criteria on the supports\footnote{Recall that a vector $\mathbf{x}$ is said to be \emph{supported} in $S$ when $\mathbf{x} \in \text{\rm span}\{\mathbf{e}_j: j\in S\}$, with $\mathbf{e}_j$ forming the standard column basis.} of sparse vectors. Let $\{1, \ldots, m\}$ be denoted $[m]$, its power set $2^{[m]}$, and ${[m] \choose k}$ the set of subsets of $[m]$ of size $k$.  A \emph{hypergraph} on vertices $[m]$  is simply any subset $\mathcal{H} \subseteq 2^{[m]}$. We say that $\mathcal{H}$ is \textit{$k$-uniform} when $\mathcal{H} \subseteq {[m] \choose k}$. The \emph{degree} $\deg_\mathcal{H}(i)$ of a node $i \in [m]$ is the number of sets in $\mathcal{H}$ that contain $i$, and we say $\mathcal{H}$ is \emph{regular} when for some $r$ we have $\deg_\mathcal{H}(i) = r$ for all $i$ (given such an $r$, we say $\mathcal{H}$ is \textit{$r$-regular}). We also write $2\mathcal{H} := \{ S \cup S': S, S' \in \mathcal{H}\}$.  The following class of structured hypergraphs is a key ingredient in this work.

\begin{definition}\label{sip}
Given $\mathcal{H} \subseteq 2^{[m]}$, the \textbf{star} $\sigma(i)$ is the collection of sets in $\mathcal{H}$ containing $i$. We say $\mathcal{H}$ has the \textbf{singleton intersection property} (\textbf{SIP}) when $\cap \sigma(i) = \{i\}$ for all $i \in [m]$. 
\end{definition}

We next give a quantitative generalization of the spark condition (\ref{SparkCondition}) to combinatorial subsets of supports. The \emph{lower bound} of an $n \times m$ matrix $\mathbf{M}$ is the largest $\alpha$ with \mbox{$\|\mathbf{M}\mathbf{x}\|_2 \geq \alpha\|\mathbf{x}\|_2$} for all $\mathbf{x} \in \mathbb{R}^m$ \cite{Grcar10}. By compactness of the unit sphere, every injective linear map has a positive lower bound; hence, if $\mathbf{M}$ satisfies \eqref{SparkCondition}, then submatrices formed from $2k$ of its columns or less have strictly positive lower bounds. 

The lower bound of a matrix is generalized below in (\ref{Ldef}) by restricting it to the spans of certain submatrices\footnote{See \cite{vidal2005generalized} for an overview of the related ``union of subspaces" model.} associated with a hypergraph $\mathcal{H} \subseteq {[m] \choose k}$ of column indices. Let $\mathbf{M}_S$ denote the submatrix formed by the columns of a matrix $\mathbf{M}$ indexed by $S \subseteq [m]$ (setting $\mathbf{M}_\emptyset := \mathbf{0}$).  In the sections that follow, we shall also let $\bm{\mathcal{M}}_S$ denote the column-span of a submatrix $\mathbf{M}_S$, and $\bm{\mathcal{M}}_\mathcal{G}$ to denote $\{\bm{\mathcal{M}}_S\}_{S \in \mathcal{G}}$.  We define:  
%
\begin{align}\label{Ldef}
L_\mathcal{H}(\mathbf{M}) := \min \left\{ \frac{\|\mathbf{M}_S\mathbf{x}\|_2}{ \sqrt{k} \|\mathbf{x}\|_2} : S \in \mathcal{H}, \ \mathbf{0} \neq \mathbf{x} \in \mathbb{R}^{|S|} \right\},
\end{align} 
writing also $L_{k}$ in place of $L_\mathcal{H}$ when $\mathcal{H} = {[m] \choose k}$.\footnote{In compressed sensing literature, \mbox{$1 - \sqrt{k} L_k(\mathbf{M})$}  is the asymmetric lower restricted isometry constant for $\mathbf{M}$ with unit $\ell_2$-norm columns \cite{Blanchard2011}.}  As explained above, compactness implies that $L_{2k}(\mathbf{M}) > 0$ for all $\mathbf{M}$ satisfying \eqref{SparkCondition}. Clearly, $L_{\mathcal{H}'}(\mathbf{M}) \geq L_\mathcal{H}(\mathbf{M})$ whenever $\mathcal{H}' \subseteq \mathcal{H}$, and similarly any $k$-uniform $\mathcal{H}$ satisfying $\cup \mathcal{H} = [m]$ has $L_2 \geq L_{2\mathcal{H}} \geq L_{2k}$ (letting $L_{2k}$ := $L_m$ whenever $2k > m$).

We are now in a position to state our main result, though for expository purposes we leave the quantity $C_1$ 
undefined until Sec.~\ref{DUT}. All results  below assume real matrices and vectors. 


\begin{theorem}\label{DeterministicUniquenessTheorem}
If an $n \times m$ matrix $\mathbf{A}$ satisfies $L_{2\mathcal{H}}(\mathbf{A}) > 0$ for some $r$-regular $\mathcal{H} \subseteq {[m] \choose k}$ with the SIP, and $k$-sparse \mbox{$\mathbf{x}_1, \ldots, \mathbf{x}_N \in \mathbb{R}^m$} include more than $(k-1){\overline m \choose k}$ vectors in general linear position\footnote{Recall that a set of vectors sharing support $S$ are in \emph{general linear position} when any $|S|$ of them are linearly independent.} supported in each $S \in \mathcal{H}$, then the following recovery guarantees hold for $C_1 > 0$ given by \eqref{Cdef1}.

\textbf{Dictionary Recovery:} Fix $\varepsilon < L_{2}(\mathbf{A}) / C_1$.\footnote{Note that the condition $\varepsilon < L_2(\mathbf{A}) /C_1$ is necessary; otherwise, with \mbox{$\mathbf{A}$ = $\mathbf{I}$} (the identity matrix) and $\mathbf{x}_i = \mathbf{e}_i$, the matrix $\mathbf{B} = \left[\mathbf{0}, \frac{1}{2}(\mathbf{e}_1 + \mathbf{e}_2), \mathbf{e}_3, \ldots, \mathbf{e}_{m} \right]$ and sparse codes $\mathbf{\overline x}_i = \mathbf{e}_2$ for $i = 1, 2$ and $\mathbf{\overline x}_i = \mathbf{e}_i$ for $i \geq 3$ satisfy $\|\mathbf{A}\mathbf{x}_i - \mathbf{B}\mathbf{\overline x}_i \|_2 \leq \varepsilon$ but nonetheless violate \eqref{Cstable}.} If an $n \times \overline m$ matrix $\mathbf{B}$ has, for every $i \in [N]$, an associated $k$-sparse $\mathbf{\overline x}_i$ satisfying \mbox{$\|\mathbf{A}\mathbf{x}_i - \mathbf{B}\mathbf{\overline x}_i\|_2 \leq \varepsilon$}, then $\overline m \geq m$, and provided that $\overline m (r-1) < mr$, there is a permutation matrix $\mathbf{P}$ and an invertible diagonal matrix $\mathbf{D}$ such that:
\begin{align}\label{Cstable}
\|\mathbf{A}_j- (\mathbf{BPD})_j\|_2 \leq C_1 \varepsilon, \ \ \text{for all } j \in J,
\end{align}
for some $J \subseteq [m]$ of size \mbox{$m - (r-1)(\overline m - m)$}. 

\textbf{Code Recovery:} If, moreover, $\mathbf{A}_J$ satisfies \eqref{SparkCondition} and $\varepsilon < L_{2k}(\mathbf{A}_J) / C_1$, then $(\mathbf{BP})_J$ also satisfies \eqref{SparkCondition} with $L_{2k}(\mathbf{BP}_J) \geq (L_{2k}(\mathbf{A}_J) - C_1 \varepsilon) / \|\mathbf{D}_J\|_1$, and for all $i \in [N]$:
\begin{align}\label{b-PDa}
\|(\mathbf{x}_i)_J - (\mathbf{D}^{-1}\mathbf{P}^{\top} \mathbf{\overline x}_i)_J\|_1 &\leq  \left( \frac{ 1+C_1 \|(\mathbf{x}_i)_J\|_1 }{ L_{2k}(\mathbf{A}_J) -  C_1\varepsilon } \right) \varepsilon,
\end{align}
where subscript $(\cdot)_J$ here represents the subvector formed from restricting to coordinates indexed by $J$.
\end{theorem}

In words, Thm.~\ref{DeterministicUniquenessTheorem} says that the smaller the regularity $r$ of the original support hypergraph $\mathcal{H}$ or the difference $\overline m - m$ between the assumed and actual number of elements in the latent dictionary, the more columns and coefficients of the original dictionary $\mathbf{A}$ and sparse codes $\mathbf{x}_i$ are guaranteed to be contained (up to noise) in the appropriately labelled and scaled recovered dictionary $\mathbf{B}$ and codes $\mathbf{\overline x}_i$, respectively. 

In the important special case when $\overline m = m$, the theorem directly implies that  $Y = \{\mathbf{Ax}_1, \ldots, \mathbf{Ax}_N\}$ has a stable $k$-sparse representation in $\mathbb{R}^m$, with inequalities \eqref{def1} guaranteed in Def.~\ref{maindef} for the following worst-case error $\varepsilon$: 
\begin{align}\label{epsdel}
\varepsilon(\delta_1, \delta_2) := \min \left\{ \frac{\delta_1}{ C_1 }, \frac{ \delta_2 L_{2k}(\mathbf{A})}{ 1 + C_1 \left( \delta_2 + \max_{i \in [N]} \|\mathbf{x}_i\|_1  \right) } \right\}.
\end{align}

Since sparse codes in general linear position are straightforward to produce with a ``Vandermonde''  construction (i.e., by choosing columns of the matrix $[\gamma_{i}^j]_{i,j=1}^{k,N}$, for distinct nonzero $\gamma_i$), we have the following direct consequence of Thm.~\ref{DeterministicUniquenessTheorem}.

\begin{corollary}\label{DeterministicUniquenessCorollary}
Given any regular hypergraph $\mathcal{H} \subseteq {[m] \choose k}$ with the SIP, there are $N =  |\mathcal{H}| \left[ (k-1){m \choose k} + 1  \right]$ vectors \mbox{$\mathbf{x}_1, \ldots, \mathbf{x}_N \in \mathbb{R}^m$} such that every matrix $\mathbf{A}$ satisfying spark condition \eqref{SparkCondition} generates $Y = \{\mathbf{A}\mathbf{x}_1, \ldots, \mathbf{A}\mathbf{x}_N\}$ with a stable $k$-sparse representation in $\mathbb{R}^m$ for $\varepsilon(\delta_1,\delta_2)$ given by \eqref{epsdel}.
\end{corollary}

One can easily verify that for every $k < m$ there are regular $k$-uniform hypergraphs $\mathcal{H}$ with the SIP besides the obvious $\mathcal{H} = {[m] \choose k}$. For instance, take $\mathcal{H}$ to be the $k$-regular set of consecutive intervals of length $k$ in some cyclic order on $[m]$. In this case, a direct consequence of Cor.~\ref{DeterministicUniquenessCorollary} is rigorous verification of the lower bound \mbox{$N = m(k-1){m \choose k} + m$} for sufficient sample size from the introduction. Special cases allow for even smaller hypergraphs. For example, if $k = \sqrt{m}$, then a 2-regular $k$-uniform hypergraph with the SIP can be constructed as the $2k$ rows and columns formed by arranging the elements of $[m]$ into a square grid.

We should stress here that framing the problem in terms of hypergraphs has allowed us to show, unlike in previous research on the subject, that the matrix $\mathbf{A}$ need not necessarily satisfy \eqref{SparkCondition} to be recoverable from data. As an example, let $\mathbf{A} = [ \mathbf{e}_1, \ldots, \mathbf{e}_5, \mathbf{v}]$ with $\mathbf{v} = \mathbf{e}_1 + \mathbf{e}_3 + \mathbf{e}_5$ and take $\mathcal{H}$ to be all consecutive pairs of indices $1, \ldots ,6$ arranged in cyclic order. Then for $k=2$, the matrix $\mathbf{A}$ fails to satisfy \eqref{SparkCondition} while still obeying the assumptions of Thm.~\ref{DeterministicUniquenessTheorem} for dictionary recovery.

A practical implication of Thm.~\ref{DeterministicUniquenessTheorem} is the following: there is an effective procedure sufficient to affirm if a proposed solution to Prob.~\ref{InverseProblem} is indeed unique (up to noise and inherent ambiguities). One need simply check that the matrix and codes satisfy the (computable) assumptions of Thm.~\ref{DeterministicUniquenessTheorem} on $\mathbf{A}$ and the $\mathbf{x}_i$. In general, however, there is no known efficient procedure. We defer a brief discussion on this point to the next section.

A less direct consequence of Thm.~\ref{DeterministicUniquenessTheorem} is the following uniqueness and stability guarantee for solutions to Prob.~\ref{SLCopt}.

\begin{theorem}\label{SLCopt}
Fix a matrix $\mathbf{A}$ and vectors $\mathbf{x}_i$ satisfying the assumptions of Thm.~\ref{DeterministicUniquenessTheorem}, only now with over \mbox{$(k-1)\left[ {\overline m \choose k} + |\mathcal{H}|k{\overline m \choose k-1}\right]$} vectors supported in general linear position in each $S \in \mathcal{H}$. Every solution to Prob.~\ref{OptimizationProblem} (with $\eta = \varepsilon/2$) satisfies recovery guarantees \eqref{Cstable} and \eqref{b-PDa} when the corresponding bounds on $\eta$ are met.
\end{theorem}

\begin{figure}
\begin{center}
\includegraphics[width=.24 \linewidth]{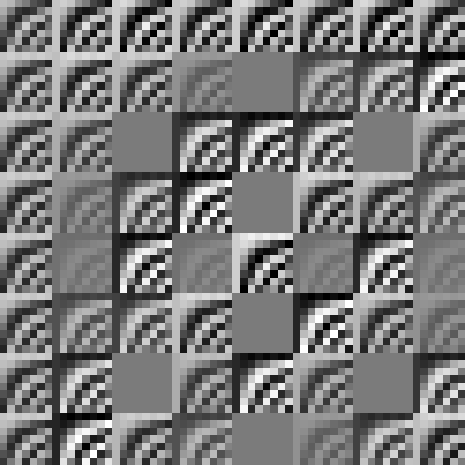}
\includegraphics[width=.24 \linewidth]{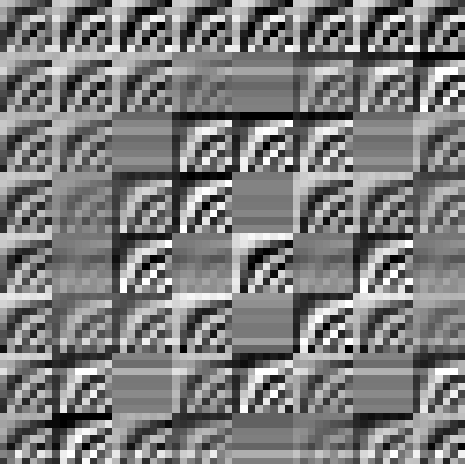}
\includegraphics[width=.24 \linewidth]{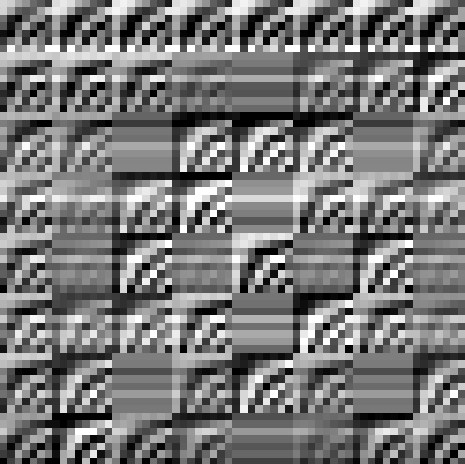}
\includegraphics[width=.24 \linewidth]{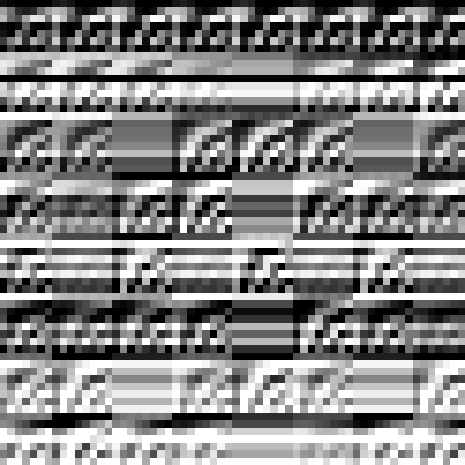}
\caption{\textbf{Learning a dictionary from increasingly noisy data}. The (unraveled) basis elements of the $8 \times 8$ discrete cosine transform (DCT) form the 64 columns of the left-most matrix above. Three increasingly imprecise dictionaries (columns reordered to best match original) are recovered by FastICA \cite{hyvarinen2000independent} trained on data generated from $8$-sparse linear combinations of DCT elements corrupted with additive noise (increasing from left to right).}
\vspace{-.6 cm}
\label{noisyrecovery}
\end{center}
\end{figure}

Another extension of Thm.~\ref{DeterministicUniquenessTheorem} can be derived from the following algebraic characterization of the spark condition.  Letting $\mathbf{A}$ be the $n \times m$ matrix of $nm$ indeterminates $A_{ij}$, the reader may work out why substituting real numbers for the $A_{ij}$ yields a matrix satisfying \eqref{SparkCondition} if and only if the following polynomial evaluates to a nonzero number:
\begin{align*}
f(\mathbf{A}) := \prod_{S \in {[m] \choose 2k}} \sum_{S' \in {[n] \choose 2k}} (\det \mathbf{A}_{S',S})^2,
\end{align*}
where for any $S' \in {[n] \choose 2k}$ and $S \in {[m] \choose 2k}$, the symbol $\mathbf{A}_{S',S}$ denotes the submatrix of entries $A_{ij}$ with $(i,j) \in S' \times S$.\footnote{The large number of terms in this product is likely necessary given that deciding whether or not a matrix satisfies the spark condition is NP-hard \cite{tillmann2014computational}.}

Since $f$ is analytic, having a single substitution of a real matrix $\mathbf{A}$ satisfying $f(\mathbf{A}) \neq 0$ implies that the zeroes of $f$ form a set of (Borel) measure zero. Such a matrix is easily constructed by adding rows of zeroes to a $\min(2k,m) \times m$ Vandermonde matrix as mentioned previously, so that every sum in the product defining $f$ above is strictly positive. Thus, almost every $n \times m$ matrix with $n \geq \min(2k,m)$ satisfies \eqref{SparkCondition}.

%


We claim that a similar phenomenon applies to datasets of vectors with a stable sparse representation. Briefly, following the same procedure as in \cite[Sec.~IV]{Hillar15}, for $k < m$ and $n \geq \min(2k, m)$, we may consider the ``symbolic'' dataset $Y = \{\mathbf{A}\mathbf{x}_1,\ldots,\mathbf{A} \mathbf{x}_N\}$ generated by an indeterminate $n \times m$ matrix $\mathbf{A}$ and $m$-dimensional $k$-sparse vectors $\mathbf{x}_1, \ldots, \mathbf{x}_N$ indeterminate within their supports, which form a regular hypergraph $\mathcal{H} \subseteq {[m] \choose k}$ satisfying the SIP. Restricting \mbox{$(k-1){m \choose k} + 1$} indeterminate $\mathbf{x}_i$ to each support in $\mathcal{H}$, and letting $\textbf{M}$ be the $n \times N$ matrix with columns $\mathbf{A}\mathbf{x}_i$, it can be checked that when $f(\mathbf{M}) \neq 0$ for a substitution of real numbers for the indeterminates, all of the assumptions on $\mathbf{A}$ and the $\mathbf{x}_i$ in Thm.~\ref{DeterministicUniquenessTheorem} are satisfied. We therefore have the following.  

\begin{theorem}\label{robustPolythm}
There is a polynomial in the entries of $\mathbf{A}$ and the $\mathbf{x}_i$ that evaluates to a nonzero number only when $Y$ has a stable $k$-sparse representation in $\mathbb{R}^m$. In particular, almost all substitutions impart to $Y$ this property.
\end{theorem}

To extend this observation to arbitrary probability distributions, note that if a set of $p$ measure spaces has all measures absolutely continuous with respect to the standard Borel measure on $\mathbb{R}$, then the product measure is also absolutely continuous with respect to the standard Borel product measure on $\mathbb{R}^p$ (e.g., see \cite{folland2013real}).  This fact combined with Thm.~\ref{robustPolythm} implies the following.\footnote{We refer the reader to \cite{Hillar15} for a more detailed explanation of these arguments.}

\begin{corollary}\label{ProbabilisticCor}
If the indeterminate entries of $\mathbf{A}$ and the $\mathbf{x}_i$ are drawn independently from probability distributions absolutely continuous with respect to the standard Borel measure, then $Y$ has a stable $k$-sparse representation in $\mathbb{R}^m$ with probability one.
\end{corollary}

Thus, drawing the dictionary and supported sparse coefficients from any continuous probability distribution almost always generates data with a stable sparse representation.


We close this section with some comments on the optimality of our results.  The linear scaling for $\varepsilon$ in \eqref{epsdel} is essentially optimal (e.g., see \cite{arias2013fundamental}), but a basic open problem remains: how many samples are necessary to determine the sparse coding model? 
Our results demonstrate that sparse codes $\mathbf{x}_i$ drawn from only a polynomial number of $k$-dimensional subspaces permit stable identification of the generating dictionary $\mathbf{A}$. 
This lends some legitimacy to the use of the model in practice, where data in general are unlikely (if ever) to exhibit the exponentially many possible $k$-wise combinations of dictionary elements required by (to our knowledge) all previously published results. 

Consequently, if $k$ is held fixed or if the size of the support set of reconstructing codes is polynomial in $\overline m$ and $k$, then a practical (polynomial) amount of data suffices to identify the dictionary.\footnote{In the latter case, a reexamination of the pigeonholing argument in the proof of Thm.~\ref{DeterministicUniquenessTheorem} requires a polynomial number of samples distributed over a polynomial number of supports.} Reasons to be skeptical that this holds in general, however, can be found in \cite{tillmann2014computational, Tillmann15}. Even so, in the next section we discuss how probabilistic guarantees can in fact be made for any number of available samples.

\section{Proofs}\label{DUT} 

We begin our proof of Thm.~\ref{DeterministicUniquenessTheorem} by showing how dictionary recovery \eqref{Cstable} already implies sparse code recovery \eqref{b-PDa} when $\mathbf{A}$ satisfies \eqref{SparkCondition} and \mbox{$\varepsilon < L_{2k}(\mathbf{A}) / C_1$}. We temporarily assume (without loss of generality) that $\overline m = m$, so as to omit an otherwise requisite subscript $(\cdot)_J$ around certain matrices and vectors. By definition of $L_{2k}$ in \eqref{Ldef}, and noting that $\sqrt{k} \|\mathbf{v}\|_2 \geq \|\mathbf{v}\|_1$ for $k$-sparse $\mathbf{v}$, we have for all $i \in [N[$:
\begin{align}\label{stuff}
\|\mathbf{x}_i - \mathbf{D}^{-1}\mathbf{P}^{\top}\mathbf{\overline x}_i \|_1 \nonumber
&\leq \frac{\|\mathbf{BPD}(\mathbf{x}_i - \mathbf{D}^{-1}\mathbf{P}^{\top}\mathbf{\overline x}_i)\|_2}{L_{2k}(\mathbf{BPD})} \\ \nonumber
&\leq \frac{\|( \mathbf{BPD} - \mathbf{A}) \mathbf{x}_i\|_2 + \|\mathbf{A}\mathbf{x}_i - \mathbf{B}\mathbf{\overline x}_i\|_2}{L_{2k}(\mathbf{BPD})} \\
&\leq \frac{C_1\varepsilon \|\mathbf{x}_i\|_1 + \varepsilon}{L_{2k}(\mathbf{BPD})},
\end{align}
where the first term in the numerator above follows from the triangle inequality and \eqref{Cstable}.

It remains for us to bound the denominator. For any $2k$-sparse $\mathbf{x}$, we have by the triangle inequality:
\begin{align*}
\|\mathbf{BPD}\mathbf{x}\|_2 
&\geq \|\mathbf{A}\mathbf{x}\|_2 - \|(\mathbf{A} - \mathbf{BPD})\mathbf{x}\|_2 \\
&\geq \sqrt{2k} (L_{2k}(\mathbf{A}) -  C_1\varepsilon) \|\mathbf{x}\|_2,
\end{align*}

%
We therefore have that $L_{2k}(\mathbf{BPD}) \geq L_{2k}(\mathbf{A}) - C_1\varepsilon  > 0$, and \eqref{b-PDa} then follows from \eqref{stuff}. The reader may also verify that $L_{2k}(\mathbf{BP}) \geq L_{2k}(\mathbf{BPD}) / \|\mathbf{D}\|_1$.

The heart of the matter is therefore \eqref{Cstable}, which we now establish beginning with the important special case of $k = 1$. 

\begin{proof}[Proof of Thm.~\ref{DeterministicUniquenessCorollary} for $k=1$]
Since the only 1-uniform hypergraph with the SIP is $[m]$, which is obviously regular, we require only $\mathbf{x}_i = c_i \mathbf{e}_i$ for $i \in [m]$, with $c_i \neq 0$ to guarantee  linear independence. While we have yet to define $C_1$ generally, in this case we may set $C_1 = 1/ \min_{\ell \in [m]} |c_{\ell}|$ so that $\varepsilon < L_2(\mathbf{A})  \min_{\ell \in [m]} |c_{\ell}|$. 

Fix $\mathbf{A} \in \mathbb{R}^{n \times m}$ satisfying $L_2(\mathbf{A}) > 0$, since here we have $2\mathcal{H} = {[m] \choose 2}$, and suppose some $\mathbf{B}$ and $1$-sparse $\mathbf{\overline x}_i \in \mathbb{R}^{\overline m}$ have  $\|\mathbf{A}\mathbf{x}_i - \mathbf{B}\mathbf{\overline x}_i\|_2 \leq \varepsilon < L_2(\mathbf{A}) / C_1$ for all $i$. Then, there exist $\overline{c}_1, \ldots, \overline{c}_m \in \mathbb{R}$ and a map $\pi: [m] \to [\overline m]$ such that:
\begin{align}\label{1D}
\|c_i\mathbf{A}_i - \overline{c}_i\mathbf{B}_{\pi(i)}\|_2 \leq \varepsilon,\ \ \text{for $i \in [m]$}.
\end{align} 
Note that $\overline{c}_i \neq 0$, since otherwise we would reach the following contradiction: $\|\mathbf{A}_i \|_2 \leq C_1 |c_i| \|\mathbf{A}_i \|_2  \leq C_1\varepsilon < L_2(\mathbf{A}) \leq L_1(\mathbf{A}) = \min_{i \in [m]} \|\mathbf{A}_{i}\|_2$. 

We now show that $\pi$ is injective (in particular, a permutation if $\overline m = m$). Suppose that $\pi(i) = \pi(j) = \ell$ for some $i \neq j$ and $\ell$. Then, $\|c_{i}\mathbf{A}_{i} - \overline{c}_{i} \mathbf{B}_{\ell}\|_2  \leq \varepsilon$ and $\|c_{j}\mathbf{A}_{j} - \overline{c}_{j}\mathbf{B}_{\ell}\|_2 \leq \varepsilon$, and we have: 
\begin{align*}
(|\overline{c}_{i}| + |\overline{c}_{j}|) \varepsilon
&\geq |\overline{c}_{i}| \|c_{j}\mathbf{A}_{j} - \overline{c}_{j}\mathbf{B}_{\ell}\|_2  + |\overline{c}_{j}| \|c_{i}\mathbf{A}_{i} - \overline{c}_{i} \mathbf{B}_{\ell}\|_2 \nonumber \\
&\geq \|\mathbf{A}(\overline{c}_{i}c_{j} \mathbf{e}_{j} - \overline{c}_{j}c_{i}\mathbf{e}_{i})\|_2 \nonumber \\ 
&\geq \sqrt{2}  L_2(\mathbf{A}) \|\overline{c}_{i}c_{j} \mathbf{e}_{j} - \overline{c}_{j}c_{i}\mathbf{e}_{i}\|_2 \nonumber \\
&\geq  L_2(\mathbf{A}) \left( |\overline{c}_{i}| + |\overline{c}_{j}| \right) \min_{\ell \in [m]} |c_\ell |,
\end{align*}
contradicting our assumed upper bound on $\varepsilon$. Hence, the map $\pi$ is injective and so $\overline m \geq m$. 

Letting $\mathbf{P}$ and $\mathbf{D}$ be the $\overline m \times \overline m$ permutation and invertible diagonal matrices with, respectively, columns $\mathbf{e}_{\pi(i)}$ and $\frac{\overline{c}_i}{c_i}\mathbf{e}_i$ for $i \in [m]$ (otherwise, $\mathbf{e}_{i}$ for $i \in [\overline{m}] \setminus [m]$), we may rewrite \eqref{1D} to see that for all $i \in [m]$:
\begin{align*}
\|\mathbf{A}_i - (\mathbf{BPD})_i\|_2 
= \|\mathbf{A}_i - \frac{\overline{c}_i}{c_i}\mathbf{B}_{\pi(i)}\|_2 
\leq \frac{\varepsilon}{|c_i|} 
\leq C_1\varepsilon.
\end{align*}
\end{proof}

An extension of the proof to the general case $k < m$ requires some additional tools to derive the general expression \eqref{Cdef1} for $C_1$. These include a generalized notion of distance (Def.~\ref{dDef}) and angle (Def.~\ref{FriedrichsDefinition}) between subspaces as well as a stability result in combinatorial matrix analysis (Lem.~\ref{MainLemma}). 

\begin{definition}\label{dDef}
For $\mathbf{u} \in \mathbb R^m$ and vector spaces $U,V \subseteq \mathbb{R}^m$, let $\text{\rm dist}(\mathbf{u}, V) := \min \{\| \mathbf{u}-\mathbf{v} \|_2: \mathbf{v} \in V\}$ and define:
\begin{align}
d(U,V) := \max_{\mathbf{u} \in U, \ \|\mathbf{u}\|_2 \leq 1} \text{\rm dist}(\mathbf{u},V).
\end{align}
\end{definition}

We note the following facts about $d$. Clearly, 
\begin{align}\label{UsubU}
U' \subseteq U \implies d(U',V) \leq d(U,V).
\end{align}
From \cite[Ch.~4 Cor.~2.6]{Kato2013}, we also have: 
\begin{align}\label{dimLem}
d(U,V) < 1 \implies \dim(U) \leq \dim(V),
\end{align}
and from \cite[Lem.~3.2]{Morris10}:
\begin{align}\label{eqdim}
\dim(U) = \dim(V) \implies d(U,V) = d(V,U).
\end{align}

The following is our result in combinatorial matrix analysis; it contains most of the complexity in the proof of Thm.~\ref{DeterministicUniquenessTheorem}. 

\begin{lemma}\label{MainLemma}
If an $n \times m$ matrix $\mathbf{A}$ has $L_{2\mathcal{H}}(\mathbf{A}) > 0$ for some $r$-regular $\mathcal{H} \subseteq {[m] \choose k}$ with the SIP, then the following holds for $C_2 > 0$ given by \eqref{Cdef2}:

Fix $\varepsilon < L_2(\mathbf{A}) / C_2$. If for some $n \times \overline m$ matrix $\mathbf{B}$ and map $\pi: \mathcal{H} \mapsto {[\overline m] \choose k}$,
\begin{align}\label{GapUpperBound}
d(\bm{\mathcal{A}}_S, \bm{\mathcal{B}}_{\pi(S)}) \leq \varepsilon, \ \  \text{for $S \in \mathcal{H}$},
\end{align}
then $\overline m \geq m$, and provided $\overline m (r-1) < mr$, there is a permutation matrix $\mathbf{P}$ and invertible diagonal $\mathbf{D}$ such that:
\begin{align}\label{MainLemmaBPD}
\|\mathbf{A}_i - (\mathbf{B}\mathbf{PD})_i\|_2 \leq C_2 \varepsilon, \ \  \text{for } i \in J,
\end{align}
for some $J \subseteq [m]$ of size \mbox{$m - (r-1)(\overline m - m)$}.
\end{lemma}



We present the constant $C_2$ (a function of $\mathbf{A}$ and $\mathcal{H}$) relative to a quantity used in \cite{Deutsch12} to analyze the convergence of the ``alternating projections" algorithm for projecting a point onto the intersection of subspaces. We incorporate this quantity into the following definition, which we refer to in our proof of Lem.~\ref{DistanceToIntersectionLemma} in the Appendix; specifically, we use it to bound the distance between a point and the intersection of subspaces given an upper bound on its distance from each subspace.

\begin{definition}\label{FriedrichsDefinition}
For a collection of real subspaces $\mathcal{V} = \{V_i\}_{i=1}^\ell$, define $\xi(\mathcal{V}) := 0$ when $|\mathcal{V}| = 1$, and otherwise:
\begin{align}\label{xi}
\xi^2(\mathcal{V}) := 1 -  \max \prod_{i=1}^{\ell-1} \sin^2  \theta \left(V_i, \cap_{j>i} V_j \right) ,
\end{align} 
where the maximum is taken over all ways of ordering 
the $V_i$ and the angle $\theta \in (0,\frac{\pi}{2}]$ is defined implicitly as \cite[Def.~9.4]{Deutsch12}:
\begin{align*}
\cos{\theta(U,W)} := \max\left\{ |\langle \mathbf{u}, \mathbf{w} \rangle|: \substack{ \mathbf{u} \in U \cap (U \cap W)^\perp, \ \|\mathbf{u}\|_2 \leq 1 \\ \mathbf{w} \in W \cap (U \cap W)^\perp, \  \|\mathbf{w}\|_2 \leq 1 } \right\}.
\end{align*}
\end{definition}
Note that $\theta \in (0,\frac{\pi}{2}]$ implies $0 \leq \xi < 1$, and that $\xi(\mathcal{V}') \leq \xi(\mathcal{V})$ when $\mathcal{V}' \subseteq \mathcal{V}$.\footnote{We acknowledge the counter-intuitive property: $\theta =  \pi/2$ when $U \subseteq W$.}  

The constant $C_2 > 0$ of Lem.~\ref{MainLemma} can now be expressed as:  
\begin{align}\label{Cdef2}
	C_2(\mathbf{A}, \mathcal{H}) := \frac{ (r+1) \max_{j \in [m]} \|\mathbf{A}_j\|_2}{ 1- \max_{\mathcal{G} \in {\mathcal{H} \choose r+1}} \xi( \bm{\mathcal{A}}_\mathcal{G} ) }.
\end{align}

We next define the constant $C_1 > 0$ of Thm.~\ref{DeterministicUniquenessTheorem} in terms of $C_2$. Given vectors $\mathbf{x}_1, \ldots, \mathbf{x}_N \in \mathbb{R}^m$, let $\mathbf{X}$ denote the $m \times N$ matrix with columns $\mathbf{x}_i$ and let $I(S)$ denote the set of indices $i$ for which $\mathbf{x}_i$ is supported in $S$.  We define:
\begin{align}\label{Cdef1}
C_1(\mathbf{A}, \mathcal{H}, \{\mathbf{x}_i\}_{i=1}^N) := \frac{ C_2(\mathbf{A}, \mathcal{H}) } { \min_{S \in \mathcal{H}} L_k(\mathbf{AX}_{I(S)}) }.
\end{align}
Given the assumptions of Thm.~\ref{DeterministicUniquenessTheorem} on $\mathbf{A}$ and the $\mathbf{x}_i$, this expression for $C_1$ is well-defined\footnote{\label{note1}To see this, fix $S \in \mathcal{H}$ and $k$-sparse $\mathbf{c}$. Using the definitions, we have $\|\mathbf{AX}_{I(S)}\mathbf{c}\|_2 \geq \sqrt{k} L_\mathcal{H}(\mathbf{A})\|\mathbf{X}_{I(S)}\mathbf{c}\|_2 \geq k L_\mathcal{H}(\mathbf{A}) L_k(\mathbf{X}_{I(S)})\|\mathbf{c}\|_2$. Thus, $L_k(\mathbf{AX}_{I(S)}) \geq \sqrt{k} L_\mathcal{H}(\mathbf{A}) L_k(\mathbf{X}_{I(S)}) > 0$, since $L_{\mathcal{H}}(\mathbf{A}) \geq L_{2\mathcal{H}}(\mathbf{A})> 0$ and $L_k(\mathbf{X}_{I(S)}) > 0$ by general linear position of the $\mathbf{x}_i$.} and yields an upper bound on $\varepsilon$ consistent with that proven sufficient in the case $k=1$ considered at the beginning of this section.\footnote{When $\mathbf{x}_i = c_i\mathbf{e}_i$, we have $C_2 \geq 2\|\mathbf{A}_i\|_2$ and the denominator in \eqref{Cdef1} becomes $\min_{i \in [m]} |c_i| \|\mathbf{A}_i\|_2$; hence, $C_1 \geq 2 / \min_{i \in [m]} |c_i|$.}

The practically-minded reader should note that the explicit constants $C_1$ and $C_2$ are effectively computable: the quantity $L_k$ may be calculated as the smallest singular value of a certain matrix, while the quantity $\xi$ involves computing ``canonical angles'' between subspaces, which reduces again to an efficient singular value decomposition. 
There is no known fast computation of $L_k$ in general, however, since even $L_{k} > 0$ is NP-hard \cite{tillmann2014computational}, although 
efficiently computable bounds have been proposed (e.g., via the ``mutual coherence" of a matrix \cite{donoho2003optimally}); alternatively, 
fixing $k$ yields polynomial complexity. Moreover, calculating $C_2$ requires an exponential number of queries to $\xi$ unless $r$ is held fixed, too (e.g., the ``cyclic order'' hypergraphs described above have $r=k$).  Thus, as presented, $C_1$ and $C_2$ are not efficiently computable in general.  

\begin{proof}[Proof of Thm.~\ref{DeterministicUniquenessCorollary} for $k < m$] 
We find a map $\pi: \mathcal{H} \to {[m] \choose k}$ for which the distance $d(\bm{\mathcal{A}}_S, \bm{\mathcal{B}}_{\pi(S)})$ is controlled by $\varepsilon$ for all $S \in \mathcal{H}$. Applying Lem.~\ref{MainLemma} then completes the proof.

Fix $S \in \mathcal{H}$. Since there are more than $(k-1){\overline m \choose k}$ vectors $\mathbf{x}_i$ supported in $S$, by the pigeonhole principle there must be some $\overline S \in {[\overline m] \choose k}$ and a set of $k$ indices $K \subseteq I(S)$ for which all $\mathbf{\overline x}_i$ with $i \in K$ are supported in $\overline S$.
It also follows\footnote{See footnote \ref{note1}.} from $L_{2\mathcal{H}}(\mathbf{A}) > 0$ and the general linear position of the $\mathbf{x}_i$ that $L_k(\mathbf{AX}_{K}) > 0$; that is, the columns of the $n \times k$ matrix $\mathbf{AX}_K$ form a basis for $\bm{\mathcal{A}}_S$. 

Fixing $\mathbf{y} \in \bm{\mathcal{A}}_S \setminus \{\mathbf{0}\}$, there then exists $\mathbf{c} = (c_1, \ldots, c_k) \in \mathbb{R}^k \setminus \{\mathbf{0}\}$ such that $\mathbf{y} = \mathbf{AX}_K\mathbf{c}$. Setting \mbox{$\mathbf{\overline{y}} = \mathbf{B\overline{X}}_K\mathbf{c}$, which is in $\bm{\mathcal{B}}_{\overline S}$}, we have by triangle inequality:
\begin{align*}
\|\mathbf{y} - \mathbf{\overline{y}}\|_2 
&= \|(\mathbf{AX}_K - \mathbf{B\overline{X}}_K)\mathbf{c}\|_2
\leq \varepsilon \|\mathbf{c}\|_1
\leq \varepsilon \sqrt{k}  \|\mathbf{c}\|_2  \\
&\leq \frac{\varepsilon}{L_k(\mathbf{AX}_K)} \|\mathbf{y}\|_2,
\end{align*}
where the last inequality uses \eqref{Ldef}. From Def.~\ref{dDef}:
\begin{align}\label{rhs222}
d(\bm{\mathcal{A}}_S, \bm{\mathcal{B}}_{\overline S}) 
\leq \frac{\varepsilon}{  L_k(\mathbf{AX}_{K}) } 
\leq \frac{\varepsilon}{  L_k(\mathbf{AX}_{I(S)}) } 
\leq \varepsilon \frac{C_1}{C_2}.
\end{align}
Finally, apply Lem.~\ref{MainLemma} with $\varepsilon < L_2(\mathbf{A})/C_1$ and $\pi(S) := \overline S$. 
\end{proof}

Before moving on to the proof of Thm.~\ref{SLCopt}, we briefly revisit our discussion on sample complexity from the end of the previous section. While an exponential number of samples may very well prove to be necessary in the deterministic or almost-certain case, our proof of Thm.~\ref{DeterministicUniquenessTheorem} can be extended to hold with some probability for \emph{any} number of samples by alternative appeal to a probabilistic pigeonholing at the point early in the proof where the (deterministic) pigeonhole principle is applied to show that for every $S \in \mathcal{H}$, there exist $k$ vectors $\mathbf{x}_i$ supported on $S$ whose corresponding $\mathbf{\overline x}_i$ all share the same support.\footnote{A famous example of such an argument is the counter-intuitive ``birthday paradox", which demonstrates that the probability of two people having the same birthday in a room of twenty-three is greater than 50\%.} 
Given insufficient samples, this argument has some less-than-certain probability of being valid for each $S \in \mathcal{H}$. Nonetheless, simulations with small hypergraphs confirm that the probability of success very quickly approaches one once the number of samples $N$ surpasses a small fraction of the deterministic sample complexity. 

\begin{proof}[Proof of Thm.~\ref{SLCopt}]
We bound the number of $k$-sparse $\mathbf{\overline x}_i$ from below and then apply Thm.~\ref{DeterministicUniquenessCorollary}. 
Let $n_p$ be the number of $\mathbf{\overline x}_i$ with $\|\mathbf{\overline x}_i\|_0 = p$.
Since the $\mathbf{x}_i$ are all $k$-sparse, by \eqref{minsum} we have:
\begin{align*}
\sum_{p=0}^{\overline m} p n_p =  \sum_{i=0}^N \|\mathbf{\overline x}_i\|_0
\leq \sum_{i=0}^N \|\mathbf{x}_i\|_0 
\leq kN
\end{align*}
Since $N = \sum_{p = 0}^{\overline m} n_p$, we then have $\sum_{p = 0}^{\overline m} (p-k) n_p \leq 0$. Splitting the sum yields:
\begin{align}\label{eqn}
\sum_{p = k+1}^{\overline m} n_p \leq \sum_{p = k+1}^{\overline m} (p-k) n_p \leq \sum_{p = 0}^k (k-p)n_p \leq k \sum_{p = 0}^{k-1} n_p,
\end{align}
demonstrating that the number of vectors $\mathbf{\overline x}_i$ that are \emph{not} $k$-sparse is bounded above by how many are $(k-1)$-sparse. 

Next, observe that no more than $(k-1)|\mathcal{H}|$ of the $\mathbf{\overline x}_i$ share a support $\overline S$ of size less than $k$. Otherwise, by the pigeonhole principle, there is some $S \in \mathcal{H}$ and a set of $k$ indices $K \subseteq I(S)$ for which all $\mathbf{x}_i$ with $i \in K$ are supported in $S$; as argued previously, \eqref{rhs222} follows. It is simple to show that $L_2(\mathbf{A}) \leq \max_j\|\mathbf{A}_j\|_2$, and since $0 \leq \xi < 1$, the right-hand side of \eqref{rhs222} is less than one for $\varepsilon < L_2(\mathbf{A})/C_1$. Thus, by \eqref{dimLem} we would have the contradiction $k = \dim(\bm{\mathcal{A}}_S) \leq \dim(\bm{\mathcal{B}}_{\overline S}) \leq |\overline S| < k.$ 

The total number of $(k-1)$-sparse vectors $\mathbf{\overline x}_i$ thus cannot exceed $|\mathcal{H}|(k-1){ \overline m \choose k-1}$. By \eqref{eqn}, no more than $|\mathcal{H}|k(k-1){ \overline m \choose k-1}$ vectors $\mathbf{\overline x}_i$ are not $k$-sparse. Since for every $S \in \mathcal{H}$ there are over $(k-1)\left[ {\overline m \choose k} + |\mathcal{H}|k{ \overline m \choose k-1} \right]$ vectors $\mathbf{x}_i$ supported there, it must be that more than $(k-1){\overline m \choose k}$ of them have corresponding $\mathbf{\overline x}_i$ that are $k$-sparse. The result now follows from Thm.~\ref{DeterministicUniquenessCorollary}, noting by the triangle inequality that $\|\mathbf{A}\mathbf{x}_i - \mathbf{B}\mathbf{\overline x}_i\| \leq 2\eta$ for $i = 1, \ldots, N$.
\end{proof}



\section{Discussion}\label{Discussion}

A main motivation for this work is the emergence of seemingly unique representations from sparse coding models trained on natural data, despite the varied assumptions underlying the many algorithms in current use. Our results constitute an important step toward explaining these phenomena as well as unifying many publications on the topic by deriving general deterministic conditions under which identification of parameters in this model is not only possible but also robust to uncertainty in measurement and model choice.

We have shown that, given sufficient data, the problem of determining a dictionary and sparse codes with minimal support size (Prob.~\ref{OptimizationProblem}) reduces to an instance of Prob.~\ref{InverseProblem}, to which our main result (Thm.~\ref{DeterministicUniquenessTheorem}) applies: every dictionary and sequence of sparse codes consistent with the data are equivalent up to inherent relabeling/scaling ambiguities and a discrepancy (error) that scales linearly with the measurement noise or modeling inaccuracy. The constants we provide are explicit and computable; as such, there is an effective procedure that sufficiently affirms if a proposed solution to these problems is indeed unique up to noise and inherent ambiguities, although it is not efficient in general.





Beyond an extension of existing noiseless guarantees \cite{Hillar15} to the noisy regime and their novel application to Prob.~\ref{OptimizationProblem}, our work contains a theory of combinatorial designs for support sets key to identification of dictionaries. We incorporate this idea into a fundamental lemma in matrix theory (Lem.~\ref{MainLemma}) that draws upon the definition of a matrix lower bound (\ref{Ldef}) induced by a hypergraph. The new insight offered by this combinatorial approach allows for guaranteed recovery of some or all dictionary elements even if: 1) dictionary size is overestimated, 2) data cover only a polynomial number of distinct sparse supports, and 3) dictionaries do not satisfy the spark condition. 

The absence of any assumptions about dictionaries solving Prob.~\ref{InverseProblem} was a major technical obstruction in proving Thm.~\ref{DeterministicUniquenessTheorem}. We sought such a general guarantee because of the practical difficulty in ensuring that an algorithm maintain a dictionary satisfying the spark condition \eqref{SparkCondition} at each iteration, an implicit requirement of all previous works except \cite{Hillar15}; indeed, even certifying a dictionary has this property is NP-hard \cite{tillmann2014computational}.

One direct application of this work is to theoretical neuroscience, wherein our theorems justify the mathematical soundness of one of the few hypothesized theories of bottleneck communication in the brain \cite{Isely10}: that sparse neural population activity is recoverable from its noisy linear compression through a randomly constructed (but unknown) wiring bottleneck by any biologically plausible unsupervised sparse coding method that solves Prob.~\ref{DeterministicUniquenessTheorem} or \ref{SLCopt} (e.g., \cite{rehnsommer2007, rozell2007neurally, pehlevan2015normative}).\footnote{We refer the reader to \cite{ganguli2012compressed} for more on interactions between dictionary learning and neuroscience.}

In fact, uniqueness guarantees with minimal assumptions apply to all areas of data science and engineering that utilize learned sparse structure. For example, several groups have applied compressed sensing to signal processing tasks; for instance, in MRI analysis \cite{lustig2008compressed}, image compression \cite{Duarte08}, and even the design of an ultrafast camera \cite{Gao14}. It is only a matter of time before these systems incorporate dictionary learning to encode and decode signals (e.g., in a device that learns structure from motion \cite{kong2016prior}), just as scientists have used sparse coding to make sense of their data \cite{jung2001imaging, agarwal2014spatially, lee2016sparse, wu2016stability}. 

Assurances offered by our theorems certify that different devices and algorithms learn equivalent representations given enough data from statistically identical systems.\footnote{To contrast with the current hot topic of ``Deep Learning'', there are few such uniqueness guarantees for these models of data; moreover, even small noise can dramatically alter their output \cite{goodfellow2014explaining}.} 
Indeed, a main reason for the sustained interest in dictionary learning as an unsupervised method for data analysis seems to be the assumed well-posedness of parameter identification in the model, confirmation of which forms the core of our findings.

We close with some challenges left open by our work. All conditions stated here guaranteeing the uniqueness and stability of sparse representations have only been shown sufficient; it remains open, therefore, to extend them to necessary conditions, be they on required sample size, the structure of support set hypergraphs, or tolerable error bounds. On this last note, we remark that our deterministic treatment considers always the ``worst-case" noise, whereas the ``effective" noise sampled from a concentrated distribution might be significantly reduced, especially for high-dimensional data. It would be of great practical benefit to see how drastically all conditions can be relaxed by requiring only probabilistic guarantees in this way, or in the spirit of our discussion on probabilistic pigeonholing to reduce sample complexity (as in the famous ``birthday paradox") following the proof of Thm.~\ref{DeterministicUniquenessTheorem}.

Another interesting question raised by our work is for which special cases is it efficient to check that a solution to Prob.~\ref{InverseProblem} or \ref{OptimizationProblem} is unique up to noise and inherent ambiguities. Considering that the sufficient conditions we have described for checking this in general are NP-hard to compute, are the necessary conditions hard? Are Probs.~\ref{InverseProblem} and \ref{OptimizationProblem} then also hard (e.g., see \cite{Tillmann15})? Finally, since Prob.~\ref{SLCopt} is intractable in general, but efficiently solvable by $\ell_1$-norm minimization when the matrix is known (and has a large enough lower bound over sparse domains \cite{eldar2012compressed}), is there a version of Thm.~\ref{SLCopt} certifying when Prob.~\ref{OptimizationProblem} can be solved efficiently in full by similar means?  

We hope these remaining challenges pique the interest of the community to pick up where we have left off and that the theoretical tools showcased here may be of use to this end.

\textbf{Acknowledgement.} We thank Fritz Sommer and Darren Rhea for early thoughts, and Ian Morris for posting \eqref{eqdim} online.

\bibliographystyle{IEEEtran}
\bibliography{chazthm_ieee_trans_sig}

\begin{thebibliography}{10}
\providecommand{\url}[1]{#1}
\csname url@samestyle\endcsname
\providecommand{\newblock}{\relax}
\providecommand{\bibinfo}[2]{#2}
\providecommand{\BIBentrySTDinterwordspacing}{\spaceskip=0pt\relax}
\providecommand{\BIBentryALTinterwordstretchfactor}{4}
\providecommand{\BIBentryALTinterwordspacing}{\spaceskip=\fontdimen2\font plus
\BIBentryALTinterwordstretchfactor\fontdimen3\font minus
  \fontdimen4\font\relax}
\providecommand{\BIBforeignlanguage}[2]{{%
\expandafter\ifx\csname l@#1\endcsname\relax
\typeout{** WARNING: IEEEtran.bst: No hyphenation pattern has been}%
\typeout{** loaded for the language `#1'. Using the pattern for}%
\typeout{** the default language instead.}%
\else
\language=\csname l@#1\endcsname
\fi
#2}}
\providecommand{\BIBdecl}{\relax}
\BIBdecl

\bibitem{Zhang15}
Z.~Zhang, Y.~Xu, J.~Yang, X.~Li, and D.~Zhang, ``A survey of sparse
  representation: algorithms and applications,'' \emph{Access, IEEE}, vol.~3,
  pp. 490--530, 2015.

\bibitem{wang2015sparse}
Z.~Wang, J.~Yang, H.~Zhang, Z.~Wang, Y.~Yang, D.~Liu, and T.~Huang,
  \emph{Sparse coding and its applications in computer vision}.\hskip 1em plus
  0.5em minus 0.4em\relax World Scientific, 2015.

\bibitem{Olshausen96}
B.~Olshausen and D.~Field, ``{Emergence of simple-cell receptive field
  properties by learning a sparse code for natural images},'' \emph{Nature},
  vol. 381, no. 6583, pp. 607--609, 1996.

\bibitem{hurri1996image}
J.~Hurri, A.~Hyv{\"a}rinen, J.~Karhunen, and E.~Oja, ``Image feature extraction
  using independent component analysis,'' in \emph{Proc. NORSIG '96 (Nordic
  Signal Proc. Symposium)}, 1996, pp. 475--478.

\bibitem{bell1997independent}
A.~Bell and T.~Sejnowski, ``The ``independent components" of natural scenes are
  edge filters,'' \emph{Vision Res.}, vol.~37, no.~23, pp. 3327--3338, 1997.

\bibitem{van1998independent}
J.~van Hateren and A.~van~der Schaaf, ``Independent component filters of
  natural images compared with simple cells in primary visual cortex,''
  \emph{Proc. R Soc. Lond. [Biol.]}, vol. 265, no. 1394, pp. 359--366, 1998.

\bibitem{donoho2001can}
D.~Donoho and A.~Flesia, ``Can recent innovations in harmonic analysis
  `explain' key findings in natural image statistics?'' \emph{Netw. Comput.
  Neural Syst.}, vol.~12, no.~3, pp. 371--393, 2001.

\bibitem{eldar2012compressed}
Y.~Eldar and G.~Kutyniok, \emph{Compressed sensing: theory and
  applications}.\hskip 1em plus 0.5em minus 0.4em\relax Cambridge University
  Press, 2012.

\bibitem{Isely10}
G.~Isely, C.~Hillar, and F.~Sommer, ``Deciphering subsampled data: adaptive
  compressive sampling as a principle of brain communication,'' in \emph{Adv.
  Neural Inf. Process. Syst.}, 2010, pp. 910--918.

\bibitem{hughes2010}
J.~Hughes, D.~Graham, and D.~Rockmore, ``Quantification of artistic style
  through sparse coding analysis in the drawings of {P}ieter {B}ruegel the
  {E}lder,'' \emph{Proc. Natl. Acad. Sci.}, vol. 107, no.~4, pp. 1279--1283,
  2010.

\bibitem{olshausen2010applied}
B.~Olshausen and M.~DeWeese, ``Applied mathematics: The statistics of style,''
  \emph{Nature}, vol. 463, no. 7284, p. 1027, 2010.

\bibitem{jung2001imaging}
T.-P. Jung, S.~Makeig, M.~McKeown, A.~Bell, T.-W. Lee, and T.~Sejnowski,
  ``Imaging brain dynamics using independent component analysis,'' \emph{Proc.
  IEEE}, vol.~89, no.~7, pp. 1107--1122, 2001.

\bibitem{agarwal2014spatially}
G.~Agarwal, I.~Stevenson, A.~Ber{\'e}nyi, K.~Mizuseki, G.~Buzs{\'a}ki, and
  F.~Sommer, ``Spatially distributed local fields in the hippocampus encode rat
  position,'' \emph{Science}, vol. 344, no. 6184, pp. 626--630, 2014.

\bibitem{lee2016sparse}
Y.-B. Lee, J.~Lee, S.~Tak, K.~Lee, D.~Na, S.~Seo, Y.~Jeong, J.~Ye, and A.~D.~N.
  Initiative, ``Sparse {SPM}: Group sparse-dictionary learning in {SPM}
  framework for resting-state functional connectivity mri analysis,''
  \emph{Neuroimage}, vol. 125, pp. 1032--1045, 2016.

\bibitem{wu2016stability}
S.~Wu, A.~Joseph, A.~Hammonds, S.~Celniker, B.~Yu, and E.~Frise,
  ``Stability-driven nonnegative matrix factorization to interpret spatial gene
  expression and build local gene networks,'' \emph{Proc. Natl. Acad. Sci.},
  vol. 113, no.~16, pp. 4290--4295, 2016.

\bibitem{Sun16}
J.~Sun, Q.~Qu, and J.~Wright, ``{Complete dictionary recovery over the sphere
  I: Overview and the geometric picture},'' \emph{IEEE Trans. Inf. Theory}, pp.
  853 -- 884, 2016.

\bibitem{Hadamard1902}
J.~Hadamard, ``Sur les probl{\`e}mes aux d{\'e}riv{\'e}es partielles et leur
  signification physique,'' \emph{Princeton University Bulletin}, vol.~13, no.
  49-52, p.~28, 1902.

\bibitem{rehnsommer2007}
M.~Rehn and F.~Sommer, ``{A network that uses few active neurones to code
  visual input predicts the diverse shapes of cortical receptive fields},''
  \emph{J. Comput. Neurosci.}, vol.~22, no.~2, pp. 135--146, 2007.

\bibitem{li2004analysis}
Y.~Li, A.~Cichocki, and S.-I. Amari, ``Analysis of sparse representation and
  blind source separation,'' \emph{Neural Comput.}, vol.~16, no.~6, pp.
  1193--1234, 2004.

\bibitem{Georgiev05}
P.~Georgiev, F.~Theis, and A.~Cichocki, ``Sparse component analysis and blind
  source separation of underdetermined mixtures,'' \emph{IEEE Trans. Neural
  Netw.}, vol.~16, pp. 992--996, 2005.

\bibitem{Aharon06}
M.~Aharon, M.~Elad, and A.~Bruckstein, ``On the uniqueness of overcomplete
  dictionaries, and a practical way to retrieve them,'' \emph{Linear Algebra
  Appl.}, vol. 416, no.~1, pp. 48--67, 2006.

\bibitem{Hillar15}
C.~Hillar and F.~Sommer, ``When can dictionary learning uniquely recover sparse
  data from subsamples?'' \emph{IEEE Trans. Inf. Theory}, vol.~61, no.~11, pp.
  6290--6297, 2015.

\bibitem{Grcar10}
J.~Grcar, ``A matrix lower bound,'' \emph{Linear Algebra Appl.}, vol. 433,
  no.~1, pp. 203--220, 2010.

\bibitem{vidal2005generalized}
R.~Vidal, Y.~Ma, and S.~Sastry, ``Generalized principal component analysis
  ({GPCA}),'' \emph{IEEE Trans. Pattern Anal. Mach. Intell.}, vol.~27, no.~12,
  pp. 1945--1959, 2005.

\bibitem{Blanchard2011}
J.~Blanchard, C.~Cartis, and J.~Tanner, ``Compressed sensing: How sharp is the
  restricted isometry property?'' \emph{SIAM Rev.}, vol.~53, no.~1, pp.
  105--125, 2011.

\bibitem{hyvarinen2000independent}
A.~Hyv{\"a}rinen and E.~Oja, ``Independent component analysis: algorithms and
  applications,'' \emph{Neural networks}, vol.~13, no. 4-5, pp. 411--430, 2000.

\bibitem{tillmann2014computational}
A.~Tillmann and M.~Pfetsch, ``The computational complexity of the restricted
  isometry property, the nullspace property, and related concepts in compressed
  sensing,'' \emph{IEEE Trans. Inf. Theory}, vol.~60, no.~2, pp. 1248--1259,
  2014.

\bibitem{folland2013real}
G.~Folland, \emph{Real analysis: modern techniques and their
  applications}.\hskip 1em plus 0.5em minus 0.4em\relax John Wiley \& Sons,
  2013.

\bibitem{arias2013fundamental}
E.~Arias-Castro, E.~Candes, and M.~Davenport, ``On the fundamental limits of
  adaptive sensing,'' \emph{IEEE Trans. Inf. Theory}, vol.~59, no.~1, pp.
  472--481, 2013.

\bibitem{Tillmann15}
A.~Tillmann, ``On the computational intractability of exact and approximate
  dictionary learning,'' \emph{IEEE Signal Process. Lett.}, vol.~22, no.~1, pp.
  45--49, 2015.

\bibitem{Kato2013}
T.~Kato, \emph{Perturbation theory for linear operators}.\hskip 1em plus 0.5em
  minus 0.4em\relax Springer Science \& Business Media, 2013, vol. 132.

\bibitem{Morris10}
I.~Morris, ``A rapidly-converging lower bound for the joint spectral radius via
  multiplicative ergodic theory,'' \emph{Adv. Math.}, vol. 225, no.~6, pp.
  3425--3445, 2010.

\bibitem{Deutsch12}
F.~Deutsch, \emph{Best approximation in inner product spaces}.\hskip 1em plus
  0.5em minus 0.4em\relax Springer Science \& Business Media, 2012.

\bibitem{donoho2003optimally}
D.~Donoho and M.~Elad, ``Optimally sparse representation in general
  (nonorthogonal) dictionaries via $\ell_1$ minimization,'' \emph{Proc. Natl.
  Acad. Sci.}, vol. 100, no.~5, pp. 2197--2202, 2003.

\bibitem{rozell2007neurally}
C.~Rozell, D.~Johnson, R.~Baraniuk, and B.~Olshausen, ``Neurally plausible
  sparse coding via thresholding and local competition,'' \emph{Neural
  Comput.}, vol.~20, no.~10, pp. 2526--2563, 2008.

\bibitem{pehlevan2015normative}
C.~Pehlevan and D.~Chklovskii, ``A normative theory of adaptive dimensionality
  reduction in neural networks,'' in \emph{Adv. Neural Inf. Process. Syst.},
  2015, pp. 2269--2277.

\bibitem{ganguli2012compressed}
S.~Ganguli and H.~Sompolinsky, ``Compressed sensing, sparsity, and
  dimensionality in neuronal information processing and data analysis,''
  \emph{Annu. Rev. Neurosci.}, vol.~35, pp. 485--508, 2012.

\bibitem{lustig2008compressed}
M.~Lustig, D.~Donoho, J.~Santos, and J.~Pauly, ``Compressed sensing {MRI},''
  \emph{IEEE Signal Process. Mag.}, vol.~25, no.~2, pp. 72--82, 2008.

\bibitem{Duarte08}
M.~Duarte, M.~Davenport, D.~Takbar, J.~Laska, T.~Sun, K.~Kelly, and
  R.~Baraniuk, ``Single-pixel imaging via compressive sampling,'' \emph{IEEE
  Signal Process. Mag.}, vol.~25, no.~2, pp. 83--91, March 2008.

\bibitem{Gao14}
L.~Gao, J.~Liang, C.~Li, and L.~Wang, ``Single-shot compressed ultrafast
  photography at one hundred billion frames per second,'' \emph{Nature}, vol.
  516, no. 7529, pp. 74--77, 2014.

\bibitem{kong2016prior}
C.~Kong and S.~Lucey, ``Prior-less compressible structure from motion,'' in
  \emph{Proc. IEEE Comput. Soc. Conf. Comput. Vis. Pattern Recognit.}, 2016,
  pp. 4123--4131.

\bibitem{goodfellow2014explaining}
I.~Goodfellow, J.~Shlens, and C.~Szegedy, ``Explaining and harnessing
  adversarial examples,'' \emph{Proc. International Conference on Learning
  Representations, 11 pp.}, 2014.

\end{thebibliography}


\section{Appendix}\label{proofs}

We prove Lem.~\ref{MainLemma} after the following auxiliary lemmas.  



\begin{lemma}\label{spanIntersectionLemma}
If $f: V \to W$ is injective, then $f\left(\cap_{i=1}^\ell V_i \right) =  \cap_{i=1}^\ell f\left(V_i\right)$ for any $V_1, \ldots, V_\ell \subseteq V$. ($f(\emptyset):=\emptyset$.)
\end{lemma}
\begin{proof}
By induction, it is enough to prove the case $\ell = 2$. Clearly, for any map $f$, if $w \in f(U \cap V)$ then $w \in f(U)$ and $w \in f(V)$; hence, $w \in f(U) \cap f(V)$. If $w \in f(U) \cap f(V)$, then $w \in f(U)$ and $w \in f(V)$; thus, $w = f(u) = f(v)$ for some $u \in U$ and $v \in V$, implying $u = v$ by injectivity of $f$. It follows that $u \in U \cap V$ and $w \in f(U \cap V)$.
\end{proof}
In particular, if a matrix $\mathbf{A}$ satisfies $L_{2\mathcal{H}}(\mathbf{A}) > 0$, then taking $V$ to be the union of subspaces consisting of vectors with supports in $2\mathcal{H}$, we have $\bm{\mathcal{A}}_{\cap \mathcal{G}} = \cap \bm{\mathcal{A}}_\mathcal{G}$ for all $\mathcal{G} \subseteq \mathcal{H}$.

\begin{lemma}\label{DistanceToIntersectionLemma}
Let $\mathcal{V} = \{V_i\}_{i=1}^k$ be a set of two or more subspaces of $\mathbb{R}^m$, and set $V = \cap \mathcal{V} $. For  $\mathbf{u} \in \mathbb{R}^m$, we have (recall Defs.~\ref{dDef}~\&~\ref{FriedrichsDefinition}):
\begin{align}\label{DTILeq}
\text{\rm dist}(\mathbf{u}, V) \leq \frac{1}{1 - \xi(\mathcal{V})} \sum_{i=1}^k \text{\rm dist}(\mathbf{u}, V_i).
\end{align}
\end{lemma}
\begin{proof} 
Recall the projection onto the subspace $V \subseteq \mathbb{R}^m$ is the mapping $\Pi_V: \mathbb{R}^m \to V$ that associates with each $\mathbf{u}$ its unique nearest point in $V$; i.e., $\|\mathbf{u} - \Pi_V\mathbf{u}\|_2 = \text{\rm dist}(\mathbf{u}, V)$.
By repeatedly applying the triangle inequality, we have:
\begin{align}\label{f}
\|\mathbf{u} - &\Pi_V\mathbf{u}\|_2 
\leq \|\mathbf{u} - \Pi_{V_k} \mathbf{u}\|_2 + \|\Pi_{V_k}  \mathbf{u} - \Pi_{V_k}\Pi_{V_{k-1}}\mathbf{u}\|_2 \nonumber \\
&\ \ \ \ \ \ \ \ \ \ \ + \cdots + \|\Pi_{V_k} \Pi_{V_{k-1}}\cdots \Pi_{V_1} \mathbf{u} - \Pi_V \mathbf{u}\|_2 \nonumber \\
&\leq  \sum_{\ell=1}^k \|\mathbf{u} - \Pi_{V_{\ell}} \mathbf{u}\|_2 
+ \|(\Pi_{V_k}\cdots\Pi_{V_{1}} - \Pi_V) \mathbf{u}\|_2,
\end{align}
where we have also used that the spectral norm of the orthogonal projections $\Pi_{V_{\ell}}$ satisfies $\|\Pi_{V_{\ell}}\|_2 \leq 1$ for all $\ell$. 

It remains to bound the second term in \eqref{f} by $\xi(\mathcal{V}) \|\mathbf{u} - \Pi_V\mathbf{u}\|_2$. First, note that $\Pi_{V_\ell} \Pi_V = \Pi_V$ and $\Pi_V^2 = \Pi_V$, so we have $\|(\Pi_{V_k} \cdots \Pi_{V_1} - \Pi_V) \mathbf{u} \|_2 
= \| ( \Pi_{V_k} \cdots\Pi_{V_1} - \Pi_V ) (\mathbf{u} - \Pi_V\mathbf{u})\|_2$. 
Consequently, inequality \eqref{DTILeq} follows from \cite[Thm.~9.33]{Deutsch12}:
\begin{align}
\|\Pi_{V_k}\Pi_{V_{k-1}}\cdots\Pi_{V_1} \mathbf{x} - \Pi_V\mathbf{x}\|_2 \leq z \|\mathbf{x}\|_2, \ \ \text{for all } \mathbf{x},
\end{align}
with \mbox{$z^2= 1 - \prod_{\ell =1}^{k-1}(1-z_{\ell}^2)$} and \mbox{$z_{\ell} = \cos\theta\left(V_{\ell}, \cap_{s=\ell+1}^k V_s\right)$} (recall $\theta$ from Def.~\ref{FriedrichsDefinition}), after substituting $\xi(\mathcal{V})$ for $z$ and rearranging terms.
\end{proof}
\begin{lemma}\label{NonEmptyLemma} 
Fix an $r$-regular hypergraph $\mathcal{H} \subseteq 2^{[m]}$ satisfying the SIP. If the map $\pi: \mathcal{H} \to 2^{[\overline m]}$ has $\sum_{S \in \mathcal{H}} |\pi(S)| \geq \sum_{S \in \mathcal{H}} |S|$ and:
\begin{align}\label{cond}
	|\cap \pi(\mathcal{G})| \leq |\cap \mathcal{G} |,\ \ \   \text{for } \mathcal{G} \in {\mathcal{H} \choose r} \cup {\mathcal{H} \choose r+1},
\end{align}
then $\overline m \geq m$; and if $\overline m  (r-1) < mr$, the map $i \mapsto \cap_{S \in \sigma(i)} \pi(S)$ is an injective function to $[\overline m]$ from some $J \subseteq [m]$ of size $m - (r-1)(\overline m - m)$ (recall $\sigma$ from Def.~\ref{sip}).  
\end{lemma}

\begin{proof}
Consider the following set: $T_1 := \{(i, S): i \in \pi(S), S \in \mathcal{H}\}$, which numbers $|T_1| = \sum_{S \in \mathcal{H}} |\pi(S)| \geq \sum_{S \in \mathcal{H}} |S| = \sum_{i \in [m]} \deg_\mathcal{H}(i) = mr$ by $r$-regularity of $\mathcal{H}$. Note that $|T_1| \leq \overline m r$; otherwise, pigeonholing the tuples of $T_1$ with respect to their $\overline m$ possible first elements would imply that more than $r$ of the tuples in $T_1$ share the same first element. This cannot be the case, however, since then some $\mathcal{G} \in {\mathcal{H} \choose r+1}$ formed from any $r+1$ of their second elements would satisfy $\cap \pi(\mathcal{G}) \neq 0$; hence, $|\cap \mathcal{G}| \neq 0$ by \eqref{cond}, contradicting $r$-regularity of $\mathcal{H}$. It follows that $\overline m \geq m$.

Suppose now that $\overline m (r-1) < mr$, so that $p := mr - \overline m (r-1)$ is positive and $|T_1| \geq \overline m (r - 1) + p$. Pigeonholing $T_1$ into $[\overline m]$ again, there are at least $r$ tuples in $T_1$ sharing some first element; that is, for some $\mathcal{G}_1 \subseteq \mathcal{H}$ of size $|\mathcal{G}_1| \geq r$, we have $|\cap \pi(\mathcal{G}_1)| \geq 1$ and (by \eqref{cond}) $|\cap \mathcal{G}_1| \geq 1$. Since no more than $r$ tuples of $T_1$ can share the same first element, we in fact have $|\mathcal{G}_1| = r$. It follows by $r$-regularity that $\mathcal{G}_1$ is a star of $\mathcal{H}$; hence, $|\cap \mathcal{G}_1| = 1$ by the SIP and $|\cap \pi(\mathcal{G}_1)|  = 1$ by \eqref{cond}.

If $p=1$, then we are done. Otherwise, define $T_2 := T_1 \setminus \{(i,S) \in T_1: i = \cap \pi(\mathcal{G}_1)\}$, which contains $|T_2| = |T_1| - r \geq (\overline m - 1)(r-1) + (p-1)$ ordered pairs having $\overline m - 1$ distinct first indices. Pigeonholing $T_2$ into $[\overline m - 1]$ and repeating the above arguments produces the star $\mathcal{G}_2 \in {\mathcal{H} \choose r}$ with intersection $\cap \mathcal{G}_2$ necessarily distinct (by $r$-regularity) from $\cap \mathcal{G}_1$. Iterating this procedure $p$ times in total yields the stars $\mathcal{G}_i$ for which $\cap\mathcal{G}_i \mapsto \cap \pi(\mathcal{G}_i)$ defines an injective map to $[\overline m]$ from $J = \{\cap \mathcal{G}_1, \ldots, \cap \mathcal{G}_p\} \subseteq [m]$.
\end{proof}

\begin{proof}[Proof of Lem.~\ref{MainLemma}]
We begin by showing that $\dim(\bm{\mathcal{B}}_{\pi(S)}) = \dim(\bm{\mathcal{A}}_S)$ for all $S \in \mathcal{H}$. Note that since $\|\mathbf{A}\mathbf{x}\|_2 \leq \max_j\|\mathbf{A}_j\|_2\|\mathbf{x}\|_1$ and $\|\mathbf{x}\|_1 \leq \sqrt{k} \|\mathbf{x}\|_2$ for all $k$-sparse $\mathbf{x}$, by \eqref{Ldef} we have $L_2(\mathbf{A}) \leq \max_j\|\mathbf{A}_j\|_2$ and therefore (as $0 \leq \xi < 1$) the right-hand side of \eqref{GapUpperBound} is less than one. From \eqref{dimLem}, we have $|\pi(S)| \geq \dim(\bm{\mathcal{B}}_{\pi(S)}) \geq \dim(\bm{\mathcal{A}}_S) = |S|$, the final equality holding by injectivity of $\mathbf{A}_S$. As $|\pi(S)| = |S|$, the claim follows. Note, therefore, that $\mathbf{B}_{\pi(S)}$ has full-column rank for all $S \in \mathcal{H}$.

We next demonstrate that \eqref{cond} holds. Fixing $\mathcal{G} \in {\mathcal{H} \choose r} \cup {\mathcal{H} \choose r+1}$, it suffices to show that $d(\bm{\mathcal{B}}_{\cap \pi(\mathcal{G})}, \bm{\mathcal{A}}_{\cap \mathcal{G}} ) < 1$, since by \eqref{dimLem} we then have $|\cap \pi(\mathcal{G})| = \dim(\bm{\mathcal{B}}_{\cap \pi(\mathcal{G})}) \leq \dim(\bm{\mathcal{A}}_{\cap \mathcal{G}}) = |\cap \mathcal{G}|$, with equalities from the full column-ranks of $\mathbf{A}_{S}$ and $\mathbf{B}_{\pi(S)}$ for all $S \in \mathcal{H}$.\footnote{Note that if ever $\bm{\mathcal{B}}_{\cap \pi(\mathcal{G})} \neq \bf 0$ while $\cap \mathcal{G} = \emptyset$, we would have $d(\bm{\mathcal{B}}_{\cap \pi(\mathcal{G})}, \bm 0 ) = 1$. However, that leads to a contradiction.} Observe that $d(\bm{\mathcal{B}}_{\cap \pi(\mathcal{G})}, \bm{\mathcal{A}}_{\cap \mathcal{G}}  ) 
\leq d\left( \cap \bm{\mathcal{B}}_{\pi(\mathcal{G})}, \cap \bm{\mathcal{A}}_\mathcal{G} \right)$ by \eqref{UsubU}, since trivially $\bm{\mathcal{B}}_{\cap \pi(\mathcal{G})} \subseteq \cap \bm{\mathcal{B}}_{\pi(\mathcal{G})}$ and also $\bm{\mathcal{A}}_{\cap \mathcal{G}} = \cap \bm{\mathcal{A}}_\mathcal{G}$ by Lem.~\ref{spanIntersectionLemma}. Recalling Def.~\ref{dDef} and applying Lem.~\ref{DistanceToIntersectionLemma} yields:
\begin{align}
d\left( \cap \bm{\mathcal{B}}_{\pi(\mathcal{G})}, \cap \bm{\mathcal{A}}_\mathcal{G} \right)
&\leq \max_{\mathbf{u} \in \cap \bm{\mathcal{B}}_{\pi(\mathcal{G})}, \ \|\mathbf{u}\|_2 \leq 1} \sum_{S \in \mathcal{G}} \frac{ \text{\rm dist}\left( \mathbf{u},\bm{\mathcal{A}}_{S} \right) }{ 1 - \xi( \bm{\mathcal{A}}_\mathcal{G} ) } \nonumber \\
&= \sum_{S \in \mathcal{G}} \frac{ d\left( \cap \bm{\mathcal{B}}_{\pi(\mathcal{G})},\bm{\mathcal{A}}_{S} \right) }{ 1 - \xi( \bm{\mathcal{A}}_\mathcal{G} ) }, \nonumber
\end{align}
passing the maximum through the sum.
Since $\cap \bm{\mathcal{B}}_{\pi(\mathcal{G})} \subseteq \bm{\mathcal{B}}_{\pi(S)}$ for all $S \in \mathcal{G}$, by \eqref{UsubU} the numerator of each term in the sum above is bounded by \mbox{$d\left( \bm{\mathcal{B}}_{\pi(S)},\bm{\mathcal{A}}_{S} \right) = d\left(\bm{\mathcal{A}}_{S}, \bm{\mathcal{B}}_{\pi(S)} \right) \leq \varepsilon$}, with the equality from \eqref{eqdim} since $\dim(\bm{\mathcal{B}}_{\pi(S)}) = \dim(\bm{\mathcal{A}}_S)$. Thus, altogether:
\begin{align}\label{last}
d(\bm{\mathcal{B}}_{\cap \pi(\mathcal{G})}, \bm{\mathcal{A}}_{\cap \mathcal{G}} )
\leq \frac{|\mathcal{G}| \varepsilon}{1 - \xi( \bm{\mathcal{A}}_\mathcal{G} )}
\leq \frac{C_2 \varepsilon}{\max_j\|\mathbf{A}_j\|_2},
\end{align}
recalling the definition of $C_2$ in \eqref{Cdef2}. Lastly, since $C_2 \varepsilon < L_2(\mathbf{A}) \leq \max_j\|\mathbf{A}_j\|_2$, we have $d(\bm{\mathcal{B}}_{\cap \pi(\mathcal{G})}, \bm{\mathcal{A}}_{\cap \mathcal{G}} ) \leq 1$ and therefore \eqref{cond} holds.

%

Applying Lem.~\ref{NonEmptyLemma}, the association $i \mapsto \cap_{S \in \sigma(i)} \pi(S)$ is an injective map $\overline \pi: J \to [\overline m]$ for some $J \subseteq [m]$ of size $m - (r-1)(\overline m - m)$, and $\mathbf{B}_{\overline \pi(i)} \neq \mathbf{0}$ for all $i \in J$ since the columns of $\mathbf{B}_{\pi(S)}$ are linearly independent for all $S \in \mathcal{H}$. Letting $\overline \varepsilon := C_2 \varepsilon / \max_i \|\mathbf{A}_i\|_2$, it follows from \eqref{eqdim} and \eqref{last} that $d\left( \bm{\mathcal{A}}_i, \bm{\mathcal{B}}_{\overline \pi(i)} \right) = d\left(\bm{\mathcal{B}}_{\overline \pi(i)},  \bm{\mathcal{A}}_i \right)  \leq \overline \varepsilon$ for all $i \in J$. 
Setting $c_i := \|\mathbf{A}_i\|_2^{-1}$ so that $\|c_i\mathbf{Ae}_i\|_2 = 1$, by Def.~\ref{dDef} for all $i \in J$:
\begin{align*}
\min_{\overline c_i \in \mathbb{R}} \|c_i\mathbf{Ae}_i - \overline c_i \mathbf{Be}_{\overline \pi(i)} \|_2
\leq d\left( \bm{\mathcal{A}}_i, \bm{\mathcal{B}}_{\overline \pi(i)} \right)
\leq \overline \varepsilon,
\end{align*}
for $\overline \varepsilon < L_2(\mathbf{A})\min_{i \in [m]}|c_i|$. But this is exactly the supposition in \eqref{1D}, with $J$ and $\overline \varepsilon$ in place of $[m]$ and $\varepsilon$, respectively. The same arguments of the case $k=1$ in Sec.~\ref{DUT} can then be made to show that for any $\overline m \times \overline m$ permutation and invertible diagonal matrices $\mathbf{P}$ and $\mathbf{D}$ with, respectively, columns $\mathbf{e}_{\pi(i)}$ and $\frac{\overline{c}_i}{c_i}\mathbf{e}_i$ for $i \in J$ (otherwise, $\mathbf{e}_{i}$ for $i \in [\overline{m}] \setminus J$), we have $\|\mathbf{A}_i - (\mathbf{B}\mathbf{PD})_i \|_2 \leq \overline  \varepsilon / |c_i|  \leq C_2 \varepsilon$ for all $i \in J$.
\end{proof}

\end{document}